\documentclass[runningheads]{llncs}

\usepackage{eccv}
\usepackage{eccvabbrv}
\usepackage{graphicx}
\usepackage{booktabs}

\usepackage[accsupp]{axessibility} 

\usepackage{hyperref}
\usepackage{orcidlink}
\usepackage{amssymb}
\usepackage{amsmath}

\usepackage{subcaption}
\usepackage{latexsym}
\usepackage{stmaryrd}
\usepackage{python}
\usepackage{makecell}
\usepackage{algorithm}
\usepackage{algpseudocode}

\DeclareMathOperator*{\argmax}{arg \, max}
\newcommand{\Mod}[1]{\ (\mathrm{mod}\ #1)}
\DeclareMathOperator*{\Lip}{Lip}
\DeclareMathOperator*{\conj}{conj}
\newcolumntype{C}[1]{>{\centering\let\newline\\\arraybackslash\hspace{0pt}}m{#1}}

\usepackage{xr}
\begin{document}

\title{Tight and Efficient Upper Bound on Spectral Norm of Convolutional Layers} 

\titlerunning{Spectral Norm of Convolution}

\author{Ekaterina Grishina\orcidlink{0009-0004-7060-0391} \and
Mikhail Gorbunov\orcidlink{0009-0006-9461-4266} \and
Maxim Rakhuba\orcidlink{0000-0001-7606-7322}}

\authorrunning{E.~Grishina et al.}
\institute{HSE University \\
\email{ergrishina@edu.hse.ru}}

\maketitle

\begin{abstract}
Controlling the spectral norm of the Jacobian matrix, which is related to the convolution operation, has been shown to improve generalization, training stability and robustness in CNNs.
Existing methods for computing the norm either tend to overestimate it or their performance may deteriorate quickly with increasing the input and kernel sizes.
In this paper, we demonstrate that the tensor version of the spectral norm of a four-dimensional convolution kernel, up to a constant factor, serves as an upper bound for the spectral norm of the Jacobian matrix associated with the convolution operation.
This new upper bound is independent of the input image resolution, differentiable and can be efficiently calculated during training. Through experiments, we demonstrate how this new bound can be used to improve the performance of convolutional architectures.

\end{abstract}

\section{Introduction}
\label{sec:intro}
Controlling spectral norm of convolutional layers' Jacobians is a way to make models more robust to input perturbations \cite{singla2021improved, singla2019fantastic, cisse2017parseval}, increase generalization performance~\cite{neyshabur2018pac}, prevent explosion of the gradients during back propagation \cite{xiao2018dynamical}. In addition, bounds on spectral norm have been used to construct orthogonal convolutional layers \cite{meunier2022dynamical, singla2021skew}. 

Despite the advantages offered by controlling the spectral norm, its efficient computation for convolutional layers remains a challenging task. Finding the spectral norm is equivalent to finding the largest singular value of a matrix and the straightforward computation of the singular value decomposition  (SVD) is not feasible due to the size of the convolution Jacobian $T$.
Most of the other existing techniques rely on the input sizes of images, which can result in a significant computational load for high-resolution images or images with more than two spatial dimensions.

In this work, we derive a new accurate upper bound for the spectral norm of the Jacobian $\|T\|_2$.
For a kernel tensor $K\in\mathbb{R}^{c_{in}\times c_{out} \times h \times w}$,  this bound can be computed with $\mathcal{O}(c_{in}c_{out}  hw)$ operations, where $c_{in}, c_{out}$ are the number of input and output channels and  $h, w$ are the filter  sizes. 
Note that this complexity does not depend on the input or output image resolution. 
As a result, controlling the spectral norm during training comes at almost no additional cost.

To be more precise, our bound is based on a norm denoted as $\|K\|_{\sigma}$, which is induced by the kernel tensor $K$ as a multilinear functional.
Specifically, we show that $\|K\|_{\sigma} \leq \|T\|_2 \leq \sqrt{hw} \|K\|_{\sigma}$.
The tensor norm $\|K\|_{\sigma}$ is well-known in applied multilinear algebra and can be efficiently computed by a higher-order generalization of a power method (HOPM)~\cite{de1995higher}. We will refer to our tensor  norm upper bound $\sqrt{hw} \|K\|_{\sigma}$ as $TN$ and will use it for several new regularization strategies.
Compared to the existing approach~\cite{singla2019fantastic}, our bound is guaranteed to be more accurate.
It is also applicable for zero-padded convolutions and can be modified to account for strided convolutions, which are both widely used in convolutional architectures. 
Moreover, it straightforwardly generalizes to convolutions with more than 2 spatial dimensions, for which efficient estimation of spectral properties is much more time-consuming.

\section{Related work}
The work \cite{sedghi2018singular} was the first to develop an algorithm for the exact computation of singular values for the circular convolution. In this method, the authors obtain~$n^2$ matrices of the shape $c_{out} \times c_{in}$ after applying the Fourier transform to a zero-padded kernel. 
Then the SVD of $n^2$ matrices is computed. 
The complexity (see \Cref{table_complexity}) grows polynomially with $n$ and can be computationally demanding for datasets with higher resolutions. 
Recently \cite{ebrahimpour2023spectrum} extended the method of \cite{sedghi2018singular} for non-circular convolutions. 
The authors of \cite{senderovich2022towards} generalized the formula proposed in \cite{sedghi2018singular} to convolutions with more than two dimensions and arbitrary strides.

Several works \cite{ryu2019plug, farniageneralizable} adapted power iteration to approximate the spectral norm of the convolution map.
In this case, time complexity also depends on the input size~$n$. The authors of \cite{araujo2021lipschitz} proposed an upper bound for the singular value of a doubly Toeplitz matrix, which can be efficiently computed, yet may be not very accurate for all filter sizes (see  \cite{delattre2023efficient}). Works \cite{delattre2023efficient, delattre2024spectral} introduced the so-called Gram iteration with superlinear convergence as opposed to the power iteration. However, it is applied to a padded kernel and, hence, \cite{delattre2023efficient} has a complexity similar to \cite{sedghi2018singular}.
Both algorithms \cite{delattre2023efficient, delattre2024spectral} can also be memory consuming, for example, in the algorithm \cite{delattre2024spectral} the spatial size of the kernel increases almost twice every iteration.

A differentiable upper bound independent of input size was introduced in \cite{singla2019fantastic} for the spectral norm of circular convolutions. It is computed as the minimum  spectral norm of four unfoldings of convolution kernel multiplied by a constant factor $\sqrt{hw}$. 
Besides being a computationally efficient upper bound, it also provides new insights into the spectral normalization of GANs \cite{miyato2018spectral} and the regularization of CNNs from~\cite{yoshida2017spectral}.
Using the Toeplitz matrix theory, \cite{yi2020asymptotic} proved that the bound also holds in the non-circular case (zero-padded convolutions). 
In this paper, we provide a new upper bound that combines the benefits of \cite{singla2019fantastic}, while being both theoretically more precise and demonstrating noticeably higher accuracy in experiments. This new bound is based on the observation that the norm of a tensor can be bounded from above by the spectral norms of its unfoldings.

\begin{table}
\centering
\caption{\label{table_complexity} Comparison of existing methods for computing the spectral norm of a convolution Jacobian. ``Acc.'' stands for ``accuracy'', ''Mem.'' for memory, ``Diff.'' for ``differentiable'', ''Ind. of $n$'' for ``independence of $n$''. Accuracy and speed are measured experimentally. Our bound is differentiable, fast, accurate and independent of the input size $n$.}
\begin{tabular}{lcccccc}
\toprule
Method & Acc. & Mem.  &  Diff.  & Ind. of $n$   & Padding  & Complexity ($\mathcal{O}$)\\
\midrule
Power method \cite{ryu2019plug, farniageneralizable} & + & + & + & - & zero & $n^2c_{out}c_{in}hw$ \\
Sedghi \etal \cite{sedghi2018singular}& + & - & - & - &circular & \makecell{\ \ $n^2c_{out}c_{in}(\log{n}+c_{in})$} \\
F4 (Fantastic four) \cite{singla2019fantastic, yi2020asymptotic} & - & + & + & + & any & $c_{out}c_{in}hw$\\
LipBound \cite{araujo2021lipschitz}& - & + & + & + & zero & $c_{out}c_{in}hw$ 
\\
Gram iteration \cite{delattre2023efficient} & + & - & + & - & circular & \makecell{\ \ $n^2c_{out}c_{in}(\log{n}+c_{in})$} \\
 PowerQR \cite{ebrahimpour2023spectrum} & + & + & - & - & zero & $n^2c_{out}c_{in}hw$\\
Tensor Norm (ours) & + & + & + & + & any & $c_{out}c_{in}hw$ \\
\bottomrule
\end{tabular}
\end{table}
\section{Preliminaries and notation}

\subsection{Background on convolutions \label{notation}}

For simplicity, let us first consider a one-dimensional convolution of a vector $X \in \mathbb{R}^n$ and a kernel vector $K \in \mathbb{R}^{w}$. 
In this paper we consider two types of convolution -- with zero and circular paddings. Thanks to the linearity of convolution, it may be expressed as a matrix-vector product $Y=TX$. In the case of zero padding with an integer parameter $p\geq 0$, we have $Y_j = \sum_i K_{p+i} X_{j+i}$ and $T$ is a matrix with constant values along diagonals -- a Toeplitz matrix. In the case of circular padding, we have $Y_j = \sum_{i}^{} K_{\lfloor \frac{w}{2}\rfloor+i} X_{(j+i)\mod n}$ and $T$ is a circulant, which is a special case of a Toeplitz matrix, where each column is a cyclic permutation of a previous one.  
Below we give examples of $T$ for $n=4, w=3, p=1$:
\[
T=\begin{bmatrix}
K_{1} & K_{2} & 0 & 0\\
K_{0}& K_{1} & K_2 & 0\\
0&  K_{0} & K_{1} & K_2\\
0 &  0 & K_{0} & K_{1}\\
\end{bmatrix} 
\left({\let\scriptstyle\textstyle \substack{\text{zero} \\ \text{padding}}}
\right),
\quad
T=\begin{bmatrix}
K_1 & K_2 & 0 & K_{0}\\
K_{0}& K_{1} & K_2 & 0\\
0&  K_{0} & K_{1} & K_2\\
K_2 &  0 & K_{0} & K_{1}\\
\end{bmatrix}
\left({\let\scriptstyle\textstyle \substack{\text{circular} \\ \text{padding}}}
\right).
\]

Let us now consider multi-channel two-dimensional convolution with stride~$1$ for ease of presentation. Larger stride values will be discussed separately. Let the convolution operation be given by a kernel tensor $K \in \mathbb{R}^{c_{out} \times c_{in}\times h \times w}$ applied to an input $X \in \mathbb{R}^{c_{in}\times n \times n}$ resulting in an output $Y \in \mathbb{R}^{c_{out}\times n \times n}$. The Jacobian of the convolution can be represented as a matrix $T \in \mathbb{R}^{c_{out}n^2 \times c_{in}n^2}$ satisfying 
\begin{equation}\label{eq:toeplmat}
vec(Y)=T vec(X),
\end{equation}
where $vec(X)$ represents reshaping into a vector with $c_{in}n^2$ entries using colexicographic order (column-major). 

Similarly to the one-dimensional case, zero padding that preserves spatial size leads to a matrix $T$ that can be represented as a doubly block Toeplitz matrix. In more detail, each of its blocks $B_k$ is a block Toeplitz matrix with unstructured dense blocks $T_{k, l}=K_{:,:,k+h_1,l+w_1}\in \mathbb{R}^{c_{out}\times c_{in}}$.
$$T = \begin{bmatrix}
B_{0} &  \dots & B_{h_2} & 0 &\dots & 0\\
\vdots &  B_{0} & \ddots & \ddots & \ddots & \vdots \\
B_{-h_1} &  \ddots & \ddots & \ddots & \ddots & 0\\
0 &  \ddots & \ddots & \ddots & \ddots & B_{h_2}\\
\vdots &  \ddots & \ddots & \ddots & B_0 & \vdots\\
0 &  \hdots & 0 & B_{-h_1} & \hdots& B_{0}\\
\end{bmatrix}, B_k = \begin{bmatrix}
T_{k,0} &  \dots & T_{k, w_2} & 0 &\dots & 0\\
\vdots &  T_{k, 0} & \ddots & \ddots & \ddots & \vdots \\
T_{k,-w_1} &  \ddots & \ddots & \ddots & \ddots & 0\\
0 &  \ddots & \ddots & \ddots & \ddots & T_{k, w_2}\\
\vdots &  \ddots & \ddots & \ddots & T_{k, 0} & \vdots\\
0 &  \hdots & 0 & T_{k, -w_1} & \hdots& T_{k, 0}\\
\end{bmatrix}$$
where $h_1, h_2$ and $w_1, w_2$ subject to $h=h_1+h_2+1, w=w_1+w_2+1$ depend on the padding size in height and width.
Depending on the order of vectorization, $T$ may be either doubly block Toeplitz (as described above) or multi-block doubly Toeplitz. Nevertheless, singular values of the Jacobian remain the same, see, e.g. \cite[Lemma 1]{yi2020asymptotic}. 

Another way to maintain the input size is by using circular padding. In this case the layer ``wraps around'' and takes pixels from the opposite side of the image when the kernel gets beyond the edges. The corresponding Jacobian is a doubly block circulant matrix. It is a special case of a doubly block Toeplitz matrix with the entries satisfying $C_{k,l}=C_{(n-l)\!\!\mod n, \ (n-k)\!\!\mod n}$.

\subsection{Matrix and tensor norms}

Let us introduce several definitions from multilinear algebra. 
First, we define the Frobenius inner product for the $d$-dimensional tensors $A, B \in \mathbb{C}^{n_1 \times \dots \times n_d}$ and the Frobenius norm:
\[ \langle A, B\rangle_F = \sum_{i_1, \dots i_d} \overline{A}_{i_1, \dots i_d} B_{i_1, \dots i_d}, \quad 
    \|A\|_F = \sqrt{\langle A, A\rangle_F}.
\]
Given also vectors $u_i \in \mathbb{C}^{n_i}$, we introduce the following multilinear functional:
\begin{equation}\label{eq:squarebr}
    \llbracket A;u_1,u_2 \dots u_d \rrbracket = \sum_{i_1, \dots i_d} A_{i_1, \dots i_d}u_{1_{i_1}}u_{2_{i_2}} \dots u_{d_{i_d}}.
\end{equation}
Spectral (second) norm of a matrix $A \in \mathbb{C}^{n_1 \times n_2}$ can be defined using this notation~\cite{golub2013matrix}: 
\begin{equation} \label{eq:nrm2_mat}
    \|A\|_2 = \sup_{\substack{{u_i\in\mathbb{C}^{n_i}\colon}\|u_i\|_2=1 \\ i=1,2}} |u_1^TAu_2| =\sup_{\substack{{u_i\in\mathbb{C}^{n_i}\colon}\|u_i\|_2=1 \\ i=1,2}}  |\llbracket A; u_1, u_2\rrbracket|.
\end{equation}
It is also well-known that $\|A\|_2$ equals to the largest singular value $\sigma_1(A)$ and the vectors $u_1,u_2$ on which the equality is attained are respectively the left and right singular vectors of $A$.
Both the largest singular value and the respective singular vectors can be computed using, e.g., the power method.

Spectral norm~\eqref{eq:nrm2_mat} naturally extends from matrices to tensors with more than two dimensions~\cite{lim2005singular}. For a $d$-dimensional tensor $A \in \mathbb{C}^{n_1 \times \dots \times n_d}$, it is defined as a norm of a multilinear functional~\eqref{eq:squarebr}:
\begin{equation} \label{eq:tensnorm}
    \|A\|_{\sigma} = \sup_{{\substack{{u_i\in\mathbb{C}^{n_i}\colon}\|u_i\|_2=1 \\ i=1,\dots,d}}}|\llbracket A;u_1,u_2 \dots u_d \rrbracket|.
\end{equation}
We will further mean that the supremum above is taken over complex vectors $u_i$ even for real $A$, which may be different from supremum over real $u_i$, see~\cref{sup:sec:B} for details. We use the notation $\|A\|_2$ for matrices and $\|A\|_{\sigma}$ for $d$-dimensional tensors when $d>2$. 

Our main result will be formulated in terms of the norm~\eqref{eq:tensnorm}. We note that this expression also defines the largest singular value $\sigma_1(A)$ of a tensor $A$ and is associated with the best rank-1 approximation problem~\cite{lim2005singular}:
\begin{equation} \label{eq:als}
    \inf_{\substack{\sigma\in\mathbb{R}_+,\, \|u_i\|_2=1 \\ i=1,\dots,d}} \|A - \sigma\, u_1 \circ \dots \circ u_d\|_F, 
\end{equation}
which admits the solution $\sigma = \sigma_1(A)$ ($\circ$ denotes tensor product). As we will discuss in more detail later, there is an analog of the power method called HOPM (Algorithm \ref{alg:cap}) that can be used to solve the best rank-1 approximation problem and calculate the spectral norm of a tensor. 
\subsection{Tensor unfoldings}
We call a matrix 
 \[
    A_{(i_1\dots i_k; j_1 \dots j_{d-k})}\in\mathbb{R}^{\left(n_{i_1}\cdot\ldots \cdot n_{i_k}\right) \times \left(n_{i_k} \cdot \ldots \cdot n_{j_{d-k}}\right)}
 \]
 an unfolding of a tensor $A \in \mathbb{R}^{n_1 \times \dots \times n_d}$, if it is obtained by first permuting indices of the tensor and then by reshaping the tensor into a matrix in a column-major order. In a Python numpy-like notation it reads as:
 \[
 \begin{split}
    & A = A\texttt{.transpose}(i_1, \dots,i_k, j_1,\dots, j_{d-k}), \\
    & A_{(i_1, i_2 \dots i_k; j_1, \dots j_{d-k})} = A\texttt{.reshape}(n_{i_1}\dots n_{i_k}, n_{j_1}\dots n_{j_{d-k}},\ \texttt{order='f'}).
\end{split}
\]

The following lemma establishes a connection between the spectral norm of a tensor and the spectral norm of its unfoldings. We will later use this result to prove the lower bound on $\|T\|_2$ and to demonstrate how our result relates to~\cite{singla2019fantastic}.
\begin{lemma}[Prop. 4.1,~\cite{wang2017operator}]\label{pr:unfolding}  
$\|K\|_\sigma \leq \|R\|_2$ for any unfolding matrix $R$ of the tensor $K$. 
\end{lemma}

\subsection{Spectral density matrix}
The approach we use to prove the main result of this paper is based on the following techniques. 
Following \cite[Theorem 1]{yi2020asymptotic}, for a doubly block Toeplitz matrix $T \in \mathbb{R}^{c_{out} n^2 \times c_{in} n^2}$ with the blocks $T_{k,l} \in \mathbb{R}^{c_{out}\times c_{in}}$, we consider a matrix-valued generating function $F\colon [-\pi, \pi]^2\to \mathbb{C}^{c_{out} \times c_{in}}$ such that 
$$T_{k, l} = \frac{1}{(2\pi)^2}\int_{-\pi}^{\pi}\int_{-\pi}^{\pi}F(\tau_1, \tau_2)e^{-i(k\tau_1+l\tau_2)} d\tau_1 d\tau_2.$$
The generating function can be explicitly written as~\cite{yi2020asymptotic}:
$$F(\tau_1, \tau_2) = \sum_{k=-h_1}^{h_2}\sum_{l=-w_1}^{w_2} T_{k, l} e^{i(k\tau_1+l\tau_2)},$$
which is a generalization of~\cite{tyrtyshnikov1996unifying,tilli1998singular}.
The function $F(\tau_1, \tau_2)$ is also called the spectral density matrix. The spectral norm of $F(\tau_1, \tau_2)$ is defined~\cite{yi2020asymptotic,tilli1998singular} as 
\[
    \|F\|_2 = \sup_{\tau_1, \tau_2 \in [-\pi, \pi]}\|F(\tau_1, \tau_2)\|_2.
\]

\begin{lemma}[Lemma 4, \cite{yi2020asymptotic}]\label{lemma:1}
$\|T\|_2 \leq \|F\|_2.$
\end{lemma}
 The set of singular values of circular convolution is obtained as ${\sigma_j(F(\tau_1, \tau_2))}$ with $(\tau_1, \tau_2) = (-\pi +\frac{2\pi j_1}{n}, -\pi +\frac{2\pi j_2}{n}), \forall j_1, j_2 \in [n]-1$ \cite{yi2020asymptotic}. However, by contrast to the circular case, for the zero-padded convolution the values of $\tau_1, \tau_2$ that lead to the singular values are not known a priori.

 \section{Main results}
 
The Lipschitz constant of a neural network $f$ with respect to the Euclidean norm for an input space $\mathcal{X}$ is defined as 
$$\Lip (f) = \sup_{x, x' \in \mathcal{X}, x\not=x'} \frac{\|f(x)-f(x')\|_2}{\|x-x'\|_2}.$$
It determines, e.g., how much the output value of the network changes relative to input perturbations. Computation of the Lipschitz constant of the whole network $f$ is a challenging task \cite{virmaux2018lipschitz}. However, assuming that $f$ is a composition of maps~$\phi_i$:
\[
    f = \phi_\ell \circ \dots \circ \phi_1,
\]
it can be bounded from above with the product of the Lipschitz constants of individual layers:
\begin{equation}\label{lip_const}
\Lip (f) \leq \prod_{i=1}^{\ell} \Lip(\phi_i).
\end{equation} 
Lipschitz constant of a linear layer $\phi(x) = Wx + b$ has the explicit representation:
\[
    \Lip (\phi) = \|W\|_2 = \sigma_1(W).
\]
In the case of a convolutional layer, we have $W =T$ from~\eqref{eq:toeplmat}.
Naive computation of its largest singular value by computing SVD appears to be intractable, as such a matrix does not even fit into GPU memory for practical sizes of layers and the computational complexity grows cubically with its size. 
Alternative iterative approaches, e.g., the power method that take the structure of $T$ into account require multiple applications of convolution, so their complexity depends on input image shape $n$.

In this work, we take benefit of the matrix structure and obtain an upper bound that is independent of $n$, valid for any padding and can be efficiently computed in practice. 
This bound is formulated in the next theorem.
\begin{theorem}\label{th:1}
Let $T \in \mathbb{R}^{c_{out} n^2 \times c_{in} n^2}$  be the Jacobian of a convolutional layer with stride $s=1$ and zero padding with the parameter $p\geq 0$ or circular padding. Let $K \in \mathbb{R}^{c_{out} \times c_{in} \times h \times w}$ be the convolution kernel. Then
$$\|K\|_{\sigma} \leq \|T\|_2 \leq  \sqrt{hw}\|K\|_{\sigma}.$$
\end{theorem}
\begin{proof}
Firstly, let us prove the upper bound. We can rewrite $F(\tau_1, \tau_2)$ as a convolution of the kernel with vectors $z_1, z_2$:
$$F(\tau_1, \tau_2) =\sum_{k=-h_1}^{h_2}\sum_{l=-w_1}^{w_2} K_{:, :, h_1+k, w_1+l} e^{i(k\tau_1+l\tau_2)} =\llbracket K; I_{c_{out}}, I_{c_{in}}, z_1, z_2\rrbracket,$$
$$z_1 = [e^{-h_1(i\tau_1)}, e^{(-h_1+1)(i\tau_1)} \dots e^{h_2(i\tau_1)}], z_2 = [e^{-w_1(i\tau_2)}, e^{(-w_1+1)(i\tau_2)} \dots e^{w_2(i\tau_2)}]$$
and for the matrices in square brackets the summation is done along their last indices. 
For complex vectors $u_i$, we have,
\begin{equation}\label{eq:th1}\begin{aligned}&\|F\|_2 = \sup_{\substack{\|u_i\|_2=1 \\ i=1,2}} |u_1^TFu_2|= \sup_{\substack{\|u_i\|_2=1 \\ i=1,2}} \left|\llbracket F; u_1, u_2 \rrbracket\right| = \sup_{\substack{\|u_i\|_2=1 \\ i=1,2}} \left|\llbracket K; u_1, u_2, z_1, z_2 \rrbracket\right| = \\
& = \sup_{\substack{\|u_i\|_2=1 \\ i=1,2}} \left|\sqrt{hw}\, \left\llbracket K; u_1, u_2, \frac{z_1}{\sqrt{h}}, \frac{z_2}{\sqrt{w}} \right\rrbracket\right| \leq \sup_{\substack{\|u_i\|_2=1 \\ i=1,\dots,4}} \left|\sqrt{hw}\llbracket K; u_1, u_2, u_3, u_4 \rrbracket\right| =\\ 
&=\sqrt{hw}\, \|K\|_{\sigma}.
\end{aligned}\end{equation}
Thus, $\|T\|_2 \leq  \sqrt{hw}\|K\|_{\sigma}$. 

 Let us consider an unfolding $R=K_{(1, 234)}$. Note that it is a submatrix of a doubly Toeplitz matrix $T$ up to a permutation of columns. In particular, we can obtain $R$ by choosing a block row in $T$ which starts with $T_{-h_1,-w_1}$ and excluding zero entries from it, see \cref{notation}. Since the spectral norm of a matrix upper bounds spectral norm of any of its submatrices, we have
\[\|K\|_{\sigma} \leq \|R\|_2 = \|K_{(1, 234)}\|_2 \leq \|T\|_2,\]
which completes the proof.
\end{proof}

\begin{remark} \label{rm:hopm}

    Note that the matrix $F$ and the vectors $z_1, z_2$ from~\eqref{eq:th1} are complex. Therefore, all suprema in this context are taken over complex vectors $u_i$. As a result, we need a complex rank-1 approximation of the real kernel $K$ to find $\|K\|_{\sigma}$ when applying HOPM \cref{alg:cap}. This is because complex and real spectral norms of a real tensor do not always coincide~\cite{friedland2018nuclear, friedland2020spectral}. In \Cref{sup:sec:B} we provide an example of a tensor for which $\sqrt{hw}\|K\|_{\sigma}$ with supremum taken over real vectors does not bound from above the spectral norm of a convolution.
    
\end{remark}
\begin{remark}\label{pr:1}
    As a direct corollary of \cref{th:1}, the bounds are exact for $h=w = 1$, meaning that $\|T\|_2=\|K\|_\sigma$.
\end{remark}

Note that Theorem~\ref{th:1} also suggests that in the worst case, our $TN$ bound does not overestimate the exact singular value of the convolution Jacobian by more than $\sqrt{hw}$ times, i.e., $\sqrt{hw}\|K\|_{\sigma} \leq \sqrt{hw} \|T\|_2$.

Let us now discuss how our bound compares to~\cite{singla2019fantastic,yi2020asymptotic}. The work \cite{singla2019fantastic} proposed to bound the singular value of circular convolution using the minimum of spectral norms of four unfoldings of the kernel:
 \begin{equation}\label{eq:f4}
\|T\|_2 \leq \sqrt{hw}\min \left(\|K_{(13;24)}\|_2, \|K_{(14,23)}\|_2, \|K_{(1;234)}\|_2, \|K_{(2;134)}\|_2 \right). \end{equation}
\Cref{pr:unfolding} shows that the tensor spectral norm never exceeds the norm of any of its unfoldings, including the ones that are not present in~\eqref{eq:f4}. Therefore, our bound is provably more accurate than the one in \cite{singla2019fantastic}.

\cref{th:1} naturally extends to convolutions with any number of dimensions. 
\begin{theorem}\label{th:higher_dim} 
Let us consider convolution with a kernel $K \in \mathbb{R}^{c_{out}\times c_{in} \times h_1 \times \ldots \times h_d}$, $d \in \mathbb{N}_+$, with arbitrary padding and stride $s=1$. 
Then
$$\|K\|_{\sigma} \leq \|T\|_2 \leq \sqrt{h_1 \dots h_d} \|K\|_{\sigma}.$$
\end{theorem}
\begin{proof}
    The proof is provided in \Cref{sup:sec:D}.
\end{proof}
The following theorem extends \cref{th:1} to the case of strided convolutions.
\begin{theorem}\label{th:2}
Let $T_s \in \mathbb{R}^{c_{out}\frac{n^2}{s^2} \times c_{in} n^2}$  be the Jacobian of a convolution with the kernel $K \in \mathbb{R}^{c_{out} \times c_{in} \times h \times w}$ and stride $s$. Then
$$\|Q\|_\sigma \leq \|T_s\|_2 \leq  \sqrt{\left\lceil\frac{h}{s} \right\rceil \left\lceil\frac{w}{s} \right\rceil}\|Q\|_{\sigma},$$
where $Q \in \mathbb{R}^{ c_{out}\times c_{in} s^2 \times \lceil\frac{h}{s} \rceil \times \lceil\frac{w}{s} \rceil}$ is padded with zeros and reshaped kernel $K$
$$K_{c, d, a, b} = Q_{c, ds^2 + s(a\Mod s)+b\Mod s, \lfloor\frac{a}{s}\rfloor, \lfloor\frac{b}{s}\rfloor}.$$
\end{theorem}

\begin{proof}
    The proof is provided in \Cref{sup:sec:A}.
\end{proof}

\section{Computation of the spectral norm}

The power method is a standard algorithm for the computation of the largest singular value and the respective singular vector of a matrix. Works \cite{de2000best, de1995higher} extended the power method from matrices to tensors of higher dimension. 
The algorithm starts with either randomly initialized vectors $u_1, u_2, \dots, u_d$, or it can start with a good approximation to the singular vectors.
For example, in our case, the kernel weights change slowly during training, so we may utilize vectors from previous iterations.
The $i$-th substep, $i=1,\dots,d$ of an iteration is equivalent to the minimization of the functional from~\eqref{eq:als} with respect to a single $u_i$. This leads to simple formulas with contractions, see \Cref{alg:cap}.
Note that according to~\Cref{rm:hopm}, we need to utilize complex vectors $u_i$, which is why we have the complex conjugate operation $\conj(u_i)$.
The operations in HOPM can be reorganized to speed up computations for the certain hardware. In our repository with code, you can find several implementations of HOPM that have proven to be faster in practice when running on either CPU or GPU.
It is also advantageous to rerun the HOPM algorithm from the beginning several times if no good initial guess is available. This will help ensure convergence to the global optimum.
Other initialization strategies also exist, see~\cite{de2000best}.

\begin{algorithm}
\caption{Higher-Order Power Method (HOPM)}\label{alg:cap}
\begin{algorithmic}
\State \textbf{Input} kernel: $K \in \mathbb{R}^{c_{out}\times c_{in}\times h \times w}$; number of iterations: $\texttt{n\_iters}$; initial unit vectors: $u_1 \in \mathbb{C}^{c_{out}}$, $u_2 \in \mathbb{C}^{c_{in}}$, $u_3 \in \mathbb{C}^{h}$, $u_4 \in \mathbb{C}^{w}$.
\State \textbf{Return} $\|K\|_{\sigma}$

\For{$1 \dots \texttt{n\_iters}$}
    \State $u_1=\llbracket K; I, u_2, u_3, u_4 \rrbracket, \, u_1 = \conj(u_1)/\|u_1\|$
    \Comment{$\mathcal{O}(c_{out}c_{in}hw)$}
    \State $u_2=\llbracket K; u_1, I, u_3, u_4 \rrbracket, \, u_2 = \conj(u_2)/\|u_2\|$
    \Comment{$\mathcal{O}(c_{out}c_{in}hw)$}
    \State $u_3=\llbracket K; u_1, u_2, I, u_4 \rrbracket, \, u_3 = \conj(u_3)/\|u_3\|$
    \Comment{$\mathcal{O}(c_{out}c_{in}hw)$}
    \State $u_4=\llbracket K; u_1, u_2, u_3, I \rrbracket, \, u_4 = \conj(u_4)/\|u_4\|$
    \Comment{$\mathcal{O}(c_{out}c_{in}hw)$}
\EndFor
\State \textbf{return} $|\llbracket K; u_1, u_2, u_3, u_4 \rrbracket|$
\Comment{$\mathcal{O}(c_{out}c_{in}hw)$}
\end{algorithmic}
\end{algorithm}

Note that, although the algorithm involves complex vectors, the norm $\|K\|_{\sigma}$ itself is a real number.
The following proposition provides explicit formulas for the norm and its gradient, which avoid complex arithmetic calculations.
\begin{proposition}
 Let $u_1, u_2, u_3, u_4$ be the complex singular vectors corresponding to $\|K\|_{\sigma}$. Let $u_j=a_j+ib_j$, where $a_j$ and $b_j$ are the real and imaginary parts. The gradient of the bound $\sqrt{hw}\|K\|_{\sigma}$ with respect to the kernel is computed as
\[\nabla_K \sqrt{hw}\|K\|_\sigma=\sqrt{hw} \nabla_K  \sqrt{real^2+im^2} = \frac{\sqrt{hw}}{\|K\|_{\sigma}} \left(real \nabla_K real + im \nabla_K  im\right),  \]
where $real$ and $im$ are the real and imaginary parts of $\llbracket K; u_1, u_2, u_3, u_4\rrbracket$.
For $part\in\{real,\, im\}$, we can write it and its gradient in terms of  predefined tensors $P_{real}, P_{im} \in \{-1, 0, 1\}^{2\times2\times2\times2}$ (see  \Cref{sup:sec:C} for their explicit representation):
\begin{equation} \label{eq:grad}
\begin{split}
    &part = \langle P_{part}, \llbracket K;\, [a_1, b_1], [a_2, b_2], [a_3, b_3], [a_4, b_4] \rrbracket \rangle_F, \\
    &\nabla_K part = \left\llbracket P_{part}; [a_1, b_1]^T, [a_2, b_2]^T, [a_3, b_3]^T, [a_4, b_4]^T \right\rrbracket.
    \end{split}
\end{equation}
\end{proposition}

\section{Experiments}\label{s:6}
The source code is available at \url{https://github.com/GrishKate/conv_norm}.
\subsection{Spectral norm computation}
\Cref{table_tightness} compares our $TN$ bound with the closest competitor, ``fantastic four'' bound \cite{singla2019fantastic} ($F4$ for short), in terms of precision. Values in each of the kernels are sampled from $\mathcal{N}(0, 1)$. We observe that $TN$ bound produces close to exact values, whereas $F4$ exceeds exact spectral norm 1.7-2.6 times. Results for strided convolutions are presented in \Cref{sup:sec:stride}.

\Cref{table_resnet} shows that the $TN$ bound provides an accurate approximation for the spectral norm of convolutional layers for the pre-trained ResNet18. Our bound performs systematically better for most layers compared to the $F4$ bound.

\Cref{sup:fig:all_methods} in Appendix demonstrates comparison with other existing methods in terms of precision, memory consumption and time performance. Although methods \cite{delattre2023efficient, sedghi2018singular} obtain a tight upper bound, they require significantly more memory than the other approaches. Algorithms based on power iteration \cite{ryu2019plug, ebrahimpour2023spectrum} are highly accurate, but their time and  memory consumption grows with the input size $n$. While the bounds \cite{singla2019fantastic, araujo2021lipschitz} are independent of $n$, they are noticeably less accurate. Our $TN$ bound does not depend on $n$ and therefore is both memory efficient and fast to compute, while being more precise than \cite{singla2019fantastic, araujo2021lipschitz}.

\begin{table}
\centering
\caption{\label{table_tightness} Comparison of the $TN$ bound and the $F4$ bound for Gaussian kernels. Reference values for spectral norm of zero-padded convolution $\|T\|_2$ were computed with power method \cite{ryu2019plug} for the image size $32 \times 32$. We used 100 iterations for all methods, stride equals to 1. }
\begin{tabular}{ l|ccc|cc}
\toprule
Kernel size  & \makecell{$\|T\|_2$ \\ (Reference)} & \makecell{$F4$  } & \makecell{$TN$\\ (Ours)} & $\dfrac{F4}{\|T\|_2}$ & $\dfrac{TN}{\|T\|_2}$ \\
\midrule
64, 64, 3, 3 & 49.35 & 81.26 & 51.51 & 1.646 & \textbf{1.044} \\
128, 128, 3, 3 & 67.14 & 114.68 & 69.99 & 1.708 & \textbf{1.042} \\
256, 256, 3, 3 & 95.71 & 164.44 & 96.47 & 1.718 & \textbf{1.008} \\
512, 512, 3, 3 & 135.62 & 234.55 & 136.99 & 1.729 & \textbf{1.01} \\
64, 64, 5, 5 & 81.55 & 175.74 & 88.25 & 2.155 & \textbf{1.082} \\
128, 128, 5, 5 & 113.5 & 250.91 & 119.29 & 2.211 & \textbf{1.051} \\
256, 256, 5, 5 & 160.07 & 354.46 & 165.37 & 2.214 & \textbf{1.033} \\
512, 512, 5, 5 & 227.1 & 501.8 & 229.66 & 2.21 & \textbf{1.011} \\
64, 64, 7, 7 & 112.94 & 293.34 & 127.69 & 2.597 & \textbf{1.131} \\
128, 128, 7, 7 & 158.36 & 415.81 & 170.99 & 2.626 & \textbf{1.08} \\
256, 256, 7, 7 & 222.45 & 587.98 & 235.28 & 2.643 & \textbf{1.058} \\
512, 512, 7, 7 & 315.25 & 834.35 & 326.2 & 2.647 & \textbf{1.035} \\
\bottomrule
\end{tabular}
\end{table}

\begin{table}
\centering
\caption{\label{table_resnet} Comparison of the $TN$ bound and the $F4$ bound for kernels from ResNet18 pretrained on ImageNet. We use 100 iterations for all methods. For layers with $h=w=1$ both bounds give an exact value, as proved in \Cref{pr:1}. The $F4$ bound does not take the stride into account, which results in even looser bound for strided convolution (see the first row, where the ratio $F4 / \|T\|_2 =3.5$).}
\begin{tabular}{lc|ccc|cc}
\toprule
Kernel size & Stride & \makecell{$\|T\|_2$ \\ (Exact)}  & \makecell{$F4$ } &  \makecell{$TN$\\ (Ours)}  & $\dfrac{F4}{\|T\|_2}$& $\dfrac{TN}{\|T\|_2}$\\
\midrule
64, 3, 7, 7 & 2 & 8.2 & 28.89 & 14.91 & 3.526 & \textbf{1.819} \\
64, 64, 3, 3 & 1 & 5.99 & 9.33 & 8.56 & 1.558 & \textbf{1.43} \\
64, 64, 3, 3 & 1 & 5.3 & 6.3 & 5.36 & 1.187 & \textbf{1.01} \\
64, 64, 3, 3 & 1 & 6.95 & 8.71 & 8.63 & 1.253 & \textbf{1.242} \\
64, 64, 3, 3 & 1 & 3.8 & 5.4 & 4.04 & 1.42 & \textbf{1.062} \\
128, 64, 3, 3 & 2 & 2.8 & 6.01 & 3.16 & 2.149 & \textbf{1.129} \\
128, 128, 3, 3 & 1 & 5.69 & 7.21 & 6.67 & 1.267 & \textbf{1.173} \\
128, 64, 1, 1 & 2 & 1.91 & 1.91 & 1.91 & 1.0 & \textbf{1.0} \\
128, 128, 3, 3 & 1 & 4.38 & 6.78 & 5.46 & 1.549 & \textbf{1.247} \\
128, 128, 3, 3 & 1 & 4.86 & 7.57 & 6.07 & 1.558 & \textbf{1.248} \\
256, 128, 3, 3 & 2 & 3.95 & 8.45 & 4.32 & 2.14 & \textbf{1.093} \\
256, 256, 3, 3 & 1 & 6.45 & 8.04 & 7.07 & 1.246 & \textbf{1.096} \\
256, 128, 1, 1 & 2 & 1.22 & 1.22 & 1.22 & 1.0 & \textbf{1.0} \\
256, 256, 3, 3 & 1 & 6.23 & 7.57 & 6.43 & 1.216 & \textbf{1.033} \\
256, 256, 3, 3 & 1 & 7.53 & 9.17 & 8.25 & 1.218 & \textbf{1.095} \\
512, 256, 3, 3 & 2 & 5.53 & 11.0 & 5.99 & 1.989 & \textbf{1.082} \\
512, 512, 3, 3 & 1 & 8.38 & 10.45 & 9.54 & 1.247 & \textbf{1.139} \\
512, 256, 1, 1 & 2 & 2.05 & 2.05 & 2.05 & 1.0 & \textbf{1.0} \\
512, 512, 3, 3 & 1 & 16.26 & 18.37 & 17.95 & 1.13 & \textbf{1.104} \\
512, 512, 3, 3 & 1 & 6.8 & 7.6 & 7.52 & 1.117 & \textbf{1.106} \\
\bottomrule
\end{tabular}
\end{table}

\subsection{Spectral norm regularization}
Following \cite{singla2019fantastic}, we apply the $TN$ bound for regularization and study its effect on image classification accuracy. We train ResNet18 on CIFAR100 and ResNet34 on ImageNet for 90 epochs. We use weight decay 1$e$-4, SGD with momentum 0.9 and initial learning rate of 0.1, which is reduced 10 times every 30 epochs. We apply the sum of estimates of the spectral norms of all convolutional layers of the network as a regularizer. The objective loss function becomes:
\[
\mathcal{L}= \mathcal{L}_{train} + \beta \sum_{i}\sigma_{i},
\]
where $\beta$ denotes the regularization coefficient, $\mathcal{L}_{train}$ is cross-entropy loss, $\sigma_i$ denotes the bound on the largest singular value of the $i^{th}$ convolutional layer computed with different methods.  During the epoch, one iteration step is sufficient to update $\sigma_i$ and the singular vectors. However, for reliability, we recompute singular vectors every epoch from random initialization. We obtain the highest accuracy with $\beta=0.0022$ for CIFAR100 and $\beta=1e$-4 for ImageNet. The results of hyperparameter tuning are presented in \Cref{sup:table_hyperparameters} in Appendix. We compare the effect of regularization both with and without weight decay for CIFAR100. \Cref{table_regularization} suggests that $TN$ reduces generalization error and gives systematically better results than $F4$.

\begin{table}
\caption{\label{table_regularization} Test accuracy with different regularizers.}
\begin{subtable}{.5\textwidth}
    \centering
    \begin{tabular}{c|cc}
    \toprule
    Method &  Acc. w/o wd & Acc. w/ wd \\
    \midrule
    Baseline & 73.10 & 73.84\\
    $F4$  & \textbf{73.96} & 74.91 \\
    $TN$ (Ours) & \textbf{73.96} & \textbf{74.99} \\
    \bottomrule
    \end{tabular}
    \caption{ResNet18 trained on CIFAR100.}
\end{subtable}%
\begin{subtable}{.5\textwidth}
    \centering
    \begin{tabular}{c|cc}
    \toprule
    Method & Acc@1 & Acc@5 \\
    \midrule
    Baseline &  73.368 & \textbf{91.438}\\
    $F4$  & 73.388 & 91.300  \\
    $TN$ (Ours) &  \textbf{73.510} & 91.420 \\
    \bottomrule
    \end{tabular}
    \caption{ResNet34 trained on ImageNet. }
\end{subtable}
\end{table}

Training time comparison is presented in \Cref{table_time_per_epoch}. The wall clock time difference between training with and without the $TN$ regularization is almost negligible. The time performance of $TN$ is comparable with $F4$ \cite{singla2019fantastic}. Methods of \cite{delattre2023efficient, sedghi2018singular} require much more memory and cause ''Out of Memory'' error for large input sizes. The singular value clipping method \cite{ebrahimpour2023spectrum} is computationally expensive and is applied only once per 50 steps. The advantage of $TN$ over the Gram iteration-based approach \cite{delattre2023efficient} is that $TN$ bound can be updated with one iteration of HOPM during training, while the Gram iteration needs to start from scratch every time. Hence, Gram iteration increases time of training multiple times, e.g., see Section 5.5 in \cite{delattre2023efficient}, where time per epoch increased by a factor of 3. By contrast, with one iteration per training step, $TN$ bound can be used as a computationally efficient regularization with a negligible increase in time and memory.

\begin{table}
\centering
\caption{\label{table_time_per_epoch} Average time per epoch for training ResNet18 on CIFAR100 with different regularizers. For computing LipBound \cite{araujo2021lipschitz} we use 10 samples. We present time for clipping once per 50 iterations with PowerQR \cite{ebrahimpour2023spectrum}. }
\begin{tabular}{c|cccccccc}
\toprule
Method & Baseline & $TN$ (Ours) & $F4$ & LipBound & PowerQR & Gram iter. & Sedghi\\
\midrule
Time (s) & 141.0 & 143.6 & 143.7 &  242.5 & 149.6 & OOM & OOM\\
\bottomrule
\end{tabular}
\end{table}

 \Cref{fig:comparison} (left) presents a comparison of our $TN$ bound with $F4$. Throughout training, the $TN$ bound appears to be more precise. Regularization with the $TN$ bound significantly decreases the singular value of convolutions and makes the bound more precise (middle). One may wonder how regularization affects the Lipschitz constant of the entire network. The right figure suggests that regularization reduces the norm of composition of convolution with the subsequent BatchNorm, meaning that the upper bound \eqref{lip_const} on the Lipschitz constant of the whole network declines. The same effect is observed  for all layers, see Figures \ref{sup:fig:1}, \ref{sup:fig:2}, \ref{sup:fig:3}.
\begin{figure}
\centering
\begin{subfigure}{1.0\textwidth}
  \includegraphics[width=\linewidth]{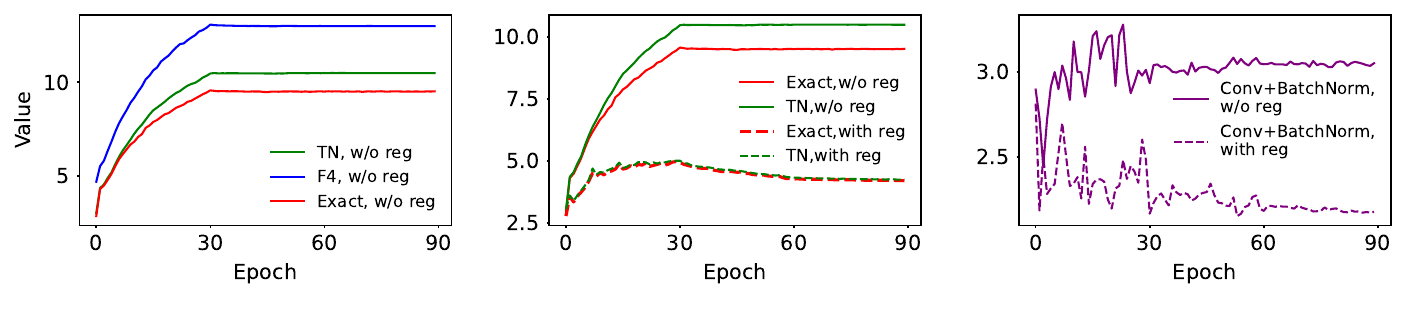}
\end{subfigure}
\caption{\label{fig:comparison} The behaviour of spectral norm bounds for the third layer of ResNet-18 during training on CIFAR100. The left figure compares tightness of $TN$ and $F4$ bounds when training without regularization. The middle figure shows the effect of training with and without $TN$ regularization. The right one demonstrates the influence of regularization on the spectral norm of composition of convolution and the subsequent BatchNorm layer. Similar plots for all layers are presented in Figures \ref{sup:fig:1}, \ref{sup:fig:2}, \ref{sup:fig:3} in Appendix. }\end{figure}

 \subsection{Orthogonal regularization}
The work \cite{wang2020orthogonal} proposes to use orthogonal regularization for training CNNs. The method utilizes convolution of the kernel with itself to estimate the divergence of the layer's Jacobian from orthogonal in the Frobenius norm, which leads to the following loss function:
\[
\mathcal{L}_{OCNN} = \|Conv(K, K) - I\|_F
\]
However, using the spectral norm of the Jacobian may seem more natural in this case, as it is an operator norm. Let $T$ be the Jacobian of the circular convolution with a kernel $K$. It is known that the product of two circulant matrices is also a circulant, thus $T^TT$ is also a Jacobian of some convolution. \cref{th:1} suggests that $\|T^TT-I\|_2 \leq \sqrt{h'w'}\|Conv(K, K, \texttt{padding=``circular''}) - I\|_\sigma$.
Thus, the following regularizer penalizes the largest difference between the singular values of the convolution and 1:
\[
\mathcal{L}_{2norm } = \|Conv(K, K, \texttt{padding=``circular''}) - I\|_{\sigma}
\]
We introduce one more regularizer based on the spectral norm of the kernel. As is known, the squared Frobenius norm of a matrix is equal to the sum of squares of its singular values and the ratio $\frac{\|A\|_2}{\|A\|_F}$ is minimized when all the singular values of matrix $A$ are equal. The Frobenius norm of the kernel can be used as approximation for the Frobenius norm of convolution Jacobian up to a constant factor, e.g., in the case of circular padding $n\|K\|_F=\|T\|_F$. Hence, we propose 
\[
\mathcal{L}_{Ratio} = \frac{\sqrt{hw}\|K\|_{\sigma}}{\|K\|_F}.
\]
We train Resnet18 on CIFAR100 dataset with the same hyperparameters as in the previous section and weight decay 1$e$-4. The regularized objective loss function is given as follows:
\[
\mathcal{L}= \mathcal{L}_{train} + \beta \sum_i \mathcal{L}^i_{reg},
\]
where $\mathcal{L}^i_{reg}$ denotes one of the regularization losses for the $i^{th}$ layer described above. The results are reported in \Cref{table_ortho}. Our regularizers improve generalization performance. $\mathcal{L}_{2norm}$ yields the best top-1 accuracy gain of more than 2\%.
\begin{table}
\centering
\caption{\label{table_ortho}ResNet18 CIFAR100 accuracy and average training time per epoch with different orthogonal regularizers.}
\begin{tabular}{C{2.2cm}|C{1.0cm}ccc}
\toprule
Method & $\beta$ & Acc@1 & Acc@5 & Time (s)\\
\midrule
Baseline & 0 &73.84 & 93.16 & 141.0 \\
OCNN & 0.1 & 75.32 & 93.45 & 150.1 \\
OCNN & 0.01& 75.32 & 93.30 & 150.1 \\
Ratio (Ours) & 1.0 & 74.98 & 93.50 & 147.0\\
Ratio (Ours) & 0.1 & 74.71 & 93.52 & 147.0\\
Ratio (Ours) & 0.01 & 74.71 & 93.45 & 147.0\\
2norm (Ours) & 1e-2 & 74.43 & 93.47 & 171.8\\
2norm (Ours) & 5e-3 & \textbf{75.99} & \textbf{94.08}  & 171.8\\
2norm (Ours) & 1e-3 & 75.83 & 93.92 & 171.8\\
\bottomrule
\end{tabular}
\end{table}

\section{Conclusion}
We propose to use the spectral norm of the kernel tensor to bound the spectral norm of convolutional layers. Our bound noticeably improves the accuracy of existing upper bounds that are independent of the resolution of input images. It also provides a trade-off between accuracy and computation efficiency compared with the other methods. 
 We demonstrate that this bound can be used as a regularizer for improving generalization of NNs. Furthermore, we propose two new regularizers based on it that enforce orthogonality of convolutions.

\section*{Acknowledgments}

The publication was supported by the grant for research centers in the field of AI provided by the Analytical Center for the Government of the Russian Federation (ACRF) in accordance with the agreement on the provision of subsidies (identifier of the agreement 000000D730321P5Q0002) and the agreement with HSE University \textnumero 70-2021-00139.
The calculations were performed in part through the computational resources of HPC facilities at HSE University~\cite{kostenetskiy2021hpc}.

\bibliographystyle{splncs04}

\newpage

\appendix

\section{Strided convolution}
\subsection{Proof of \Cref{main:th:2} \label{sec:A}}
\begin{proof}
The proof uses techniques of the proof of Theorem 2 in \cite{senderovich2022towards}.
In the case of convolution with zero padding, let $J_k \in \mathbb{R}^{\frac{n}{s} \times \frac{n}{s}}$ denote a matrix with ones along the $k$-th diagonal (if $k<0$ the ones are below the main diagonal). For periodic convolution, $J_k=P^k$, where $P$ is a permutation matrix: 
\begin{equation}\label{eq:permutation}P=\begin{bmatrix}
    0 & 1 & \hdots & 0\\
    \vdots & \vdots & \ddots  & \vdots\\
    0 & 0 & \hdots & 1 \\
    1 & 0 & \hdots & 0 \\ 
\end{bmatrix}.\end{equation}
The Jacobian $T_1$ of a convolution with a stride $s=1$ can be expressed as a summation of block doubly Toeplitz matrices:
\[T_1= \sum_{k=-h_1}^{h_2}\sum_{p=-w_1}^{w_2}J_k \otimes J_p \otimes K_{:, :, k, p}.\]
 We define a ``strided Toeplitz matrix'' as a Toeplitz matrix where each row is shifted by $s > 1$, resulting in a shape of $\frac{n}{s} \times n$. For example, when $n=6$ and $s=2$, the strided Toeplitz matrix takes the form:
\[A=\begin{bmatrix}
a & b & c & d & 0 & 0\\
e & f & a & b & c & d\\
0 & 0 & e & f & a & b
\end{bmatrix}.\]
Slicing this matrix by selecting columns at intervals determined by the stride s, such as columns $(0, s, 2\cdot s, \dots (\frac{n}{s}-1)\cdot s )$, results in a standard Toeplitz matrix. We can write this property as
\[A=\sum_{j=0}^{s-1}A_{j}(I_{\frac{n}{s}} \otimes e_j^T).\]
Let $B_i \in \mathbb{R}^{\frac{n^2}{s^2}c_{out} \times \frac{n^2}{s}c_{in}}$ be a block Toeplitz matrix. Each of its blocks $ B_{ik} \in \mathbb{R}^{\frac{n}{s}c_{out} \times c_{in}n}$ is a block Toeplitz matrix with stride $s$. As in the previous example, we can denote slices of $B_{ik}$ as $B_{ikj} \in \mathbb{R}^{c_{out}\frac{n}{s} \times c_{in}\frac{n}{s}}$, where each $B_{ikj}$ is a block Toeplitz matrix with blocks of size $c_{out} \times c_{in}$. For brevity, we denote $k_1 = \lceil \frac{-h_1-i}{s}\rceil, k_2 = \lfloor \frac{h_2-i}{s}\rfloor, p_1 = \lceil \frac{-w_1-i}{s}\rceil, p_2 = \lfloor \frac{w_2-i}{s}\rfloor$.
\[\begin{aligned} B_i &=\sum_{k=k_1}^{k_2} J_k \otimes B_{ik} = \sum_{k=k_1}^{k_2} J_k \otimes \sum_{j=0}^{s-1} B_{ikj} ((I_{\frac{n}{s}} \otimes e_i^T)\otimes I_{c_{in}}) =\\
&=\sum_{k=k_1}^{k_2} \sum_{j=0}^{s-1} J_k \otimes B_{ijk} \otimes  e_j^T.\end{aligned}\]
Strided linear transformation matrix may be expressed as follows
\[\begin{aligned}T_s& = \sum_{i=0}^{s-1}B_i ((I_{\frac{n}{s}} \otimes e_i^T)\otimes I_{c_{in}n}) = \sum_{i=0}^{s-1}B_i \otimes e_i^T  = \\
&=\sum_{i=0}^{s-1}\sum_{j=0}^{s-1}\sum_{k=k_1}^{k_2}   J_k \otimes B_{ikj} \otimes  e_j^T \otimes e_i^T = \\
& = \sum_{i=0}^{s-1}\sum_{j=0}^{s-1}\sum_{k=k_1}^{k_2} \sum_{p=p_1}^{p_2}  J_k \otimes J_p \otimes K_{:, :, i+ks, j+ps} \otimes  e_j^T \otimes e_i^T \\
&= \sum_{q=0}^{s^2-1}\sum_{k=\lceil \frac{-h_1-q/s}{s}\rceil}^{\lfloor \frac{-h_2-q/s}{s}\rfloor} \sum_{p=\lceil \frac{-w_1-q\%s}{s}\rceil}^{\lfloor \frac{-w_2-q\%s}{s}\rfloor} J_k \otimes J_p \otimes K_{:, :, \lfloor \frac{q}{s}\rfloor +ks, q \Mod s +ps} \otimes  e_q^T =\\
& = \sum_{k=\lceil \frac{-h_1-s+1}{s}\rceil}^{\lfloor \frac{h_2}{s}\rfloor} \sum_{p=\lceil \frac{-w_1-s+1}{s}\rceil}^{\lfloor \frac{w_2}{s}\rfloor}  J_k \otimes J_p \otimes Q_{:, :, k, p}.
\end{aligned}\]
Thus, strided convolution may be viewed as a convolution with another kernel $Q \in \mathbb{R}^{ c_{out}\times c_{in} s^2 \times \lceil\frac{h}{s} \rceil \times \lceil\frac{w}{s} \rceil}$, where
\[K_{c, d, a, b} = Q_{c, ds^2 + s(a\Mod s)+b\Mod s, \lfloor\frac{a}{s}\rfloor, \lfloor\frac{b}{s}\rfloor}.\]
From \Cref{main:th:1}, it follows that $\|Q\|_{\sigma} \leq \|T_s\|_2 \leq \sqrt{\lceil \frac{h}{s}\rceil \lceil \frac{w}{s}\rceil}\|Q\|_{\sigma}$.

Here is an example of the code in pytorch for obtaining $Q$ from $K$:
\begin{verbatim}
cout, cin, h, w = K.shape
if s != 1:
    if h % s != 0 and w % s != 0:
          p = (0, s - h % s, 0, s - w % s)
          K = F.pad(K, p, 'constant' , 0)
    Q = K.reshape(cout, cin, ceil(h/s), s, ceil(w/s), s)
    Q = Q.permute(0, 1, 3, 5, 2, 4) 
    Q = Q.reshape(cout, cin*s*s, ceil(h/s), ceil(w/s))
else:
    Q = K
\end{verbatim}
\end{proof}

\subsection{Experiments with a strided convolution \label{sec:stride}}

The bounds computed using \Cref{main:th:2} for a strided convolution with various kernel sizes are presented in the \Cref{table_stride}. We observe that accuracy of the bounds depends on the size of the padding. Large number of zeros added to the kernel during padding reduces the accuracy. 
\begin{table}
\centering
\caption{\label{table_stride} Comparison of the $TN$ bound and the $F4$ bound for a convolution with stride. The kernels are sampled from the Gaussian distribution. The exact value was computed with the power method \cite{ryu2019plug} for the input size $32 \times 32$. We use 150 iterations for all the methods.}
\begin{tabular}{ lc|ccc|cc}
\toprule
Sizes & Stride & \makecell{$\|T\|_2$ } & \makecell{$F4$ } & \makecell{$TN$\\ (Ours) } & $\dfrac{F4}{\|T\|_2}$ & $\dfrac{TN}{\|T\|_2}$ \vspace{0.5mm}\\
\midrule
64, 64, 3, 3 & 2 & 38.14 & 53.85 & 46.72 & 1.412 & \textbf{1.225} \\
512, 512, 3, 3 & 2 & 106.85 & 153.33 & 134.87 & 1.435 & \textbf{1.262} \\
64, 64, 5, 5 & 2 & 60.9 & 106.91 & 72.69 & 1.756 & \textbf{1.194} \\
512, 512, 5, 5 & 2 & 172.57 & 301.7 & 204.19 & 1.748 & \textbf{1.183} \\
64, 64, 7, 7 & 2 & 86.51 & 175.41 & 100.3 & 2.028 & \textbf{1.16} \\
512, 512, 7, 7 & 2 & 238.99 & 495.99 & 273.23 & 2.075 & \textbf{1.143} \\
64, 64, 3, 3 & 4 & 31.93 & 31.93 & 31.93 & 1.0 & \textbf{1.0} \\
512, 512, 3, 3 & 4 & 90.12 & 89.98 & 90.12 & 0.998 & \textbf{1.0} \\
64, 64, 5, 5 & 4 & 50.94 & 80.58 & 79.09 & 1.582 & \textbf{1.553} \\
512, 512, 5, 5 & 4 & 145.03 & 228.49 & 225.15 & 1.575 & \textbf{1.552} \\
64, 64, 7, 7 & 4 & 70.67 & 86.77 & 79.13 & 1.228 & \textbf{1.12} \\
512, 512, 7, 7 & 4 & 199.65 & 246.55 & 225.43 & 1.235 & \textbf{1.129} \\
\bottomrule
\end{tabular}
\end{table}

\subsection{Spectral density function for strided convolutions}

\Cref{main:th:2} shows that strided convolution with a kernel $K$ can be expressed as a regular convolution with a kernel $Q$. Therefore, it admits a  representation as a block doubly Toeplitz matrix with $T_{k, l}$:
\[
    T_{k, l} = Q_{:, :, k+h_1, l+w_1}
\]

Additionally, by permuting output channel dimension of kernel $Q$ (grouping modulo $s^2$), $T_{k, l}$ becomes a concatenation of $s^2$ matrices $T^{1}_{k, l}, \dots, T^{s^2}_{k, l}$, where $T^i$ correspond to some convolutions with possibly different kernel sizes (up to padding). Since the padding of kernels does not affect the spectral density function, spectral density of such a matrix is a concatenation of $s^2$ spectral density function with possibly different kernel sizes:
\[
    F(\omega_1, \omega_2) = 
    \begin{pmatrix}
        F^{1}(\omega_1, \omega_2) & \dots & F^{s^2}(\omega_1, \omega_2)
    \end{pmatrix}
\]

Contrary to the more sophisticated spectral density function proposed in Section VI.A in \cite{yi2020asymptotic}, \Cref{main:lemma:1} also holds true since $Q$ is a convolution kernel. This shows that bounding the spectral norm of $T$ can be alternatively done on $F$, which can allow for further analysis of strided convolutions.

\section{Real and complex rank-1 approximation of tensors \label{sec:B}}
In this section, we will present an example of a real tensor, for which real and complex best rank-1 approximations do not coincide with each other.  Let us use notation $\|K\|_{\sigma,\mathbb{C}}$ in the case when supremum in the definition of the spectral norm is taken over complex vectors, and $\|K\|_{\sigma,\mathbb{R}}$ in the case of real vectors. The $\sqrt{hw}\|K\|_{\sigma,\mathbb{C}}$ bounds from above  the spectral norm of a convolution for any input size, while this is not the case for $\sqrt{hw}\|K\|_{\sigma,\mathbb{R}}$ as we show in this section. Our example was inspired by \cite{draisma2018best, de2008tensor}, see also examples in \cite{friedland2018nuclear, friedland2020spectral}.

Let a tensor $K \in \mathbb{R}^{2\times2\times2\times2}$ be $K = (e_1+ie_2)^{\circ4} + (e_1-ie_2)^{\circ4}$, where $e_1=(1,0)^T, e_2=(0,1)^T$, $\circ$ denotes tensor product. $K$ is a real-valued tensor with an unfolding
\[K.reshape(2, 8, \texttt{order='c'})=\begin{bmatrix}
2 & 0 & 0 & -2 & 0 & -2 & -2 & 0\\
0 & -2 & -2 & 0 & -2 & 0 & 0 & 2\\
\end{bmatrix}.\]
$K$ is a supersymmetric tensor, which implies that it has a symmetric real best rank-1 approximation \cite{friedland2013best}. Let $x$ be 
\[x = \argmax_{\|x\|_2=1,\, x\in\mathbb{R}^2} |\llbracket K,x,x,x,x\rrbracket|.\]
Thus, we can use the parametrization $x=(\cos\alpha, \sin\alpha)^T$.
\[\begin{aligned}|\llbracket K,x,x,x,x\rrbracket| &= |(x^T(e_1+ie_2))^4 + (x^T(e_1-ie_2))^4|=\\&=|(\cos\alpha+i\sin\alpha)^4 + (\cos\alpha-i\sin\alpha)^4 |=\\
&=|\cos(4\alpha)+i\sin(4\alpha) +\cos(4\alpha)-i\sin(4\alpha)|= \\&=2|\cos(4\alpha)| \leq 2\end{aligned}\]
Let $\alpha=\frac{\pi}{4}$ and $x=\left(\frac{1}{\sqrt{2}}, \frac{1}{\sqrt{2}}\right)$, then the best real-valued rank-1 approximation of $K$ is a tensor $2x^{\circ4}$ and the corresponding singular value is $\|K\|_{\sigma,\mathbb{R}}=2$. Hence, the bound is $\sqrt{hw}\|K\|_{\sigma,\mathbb{R}}=4$.

The spectral norm of the above unfolding is $\|K.reshape(2, -1)\|_{2}=4$. From \Cref{main:pr:unfolding}, we conclude that $\|K\|_{\sigma,\mathbb{C}} \leq \|K.reshape(2, -1)\|_2=4$. The complex rank-1 approximation of $K$ is a tensor $4x^{\circ4}$, where $x=\left(\frac{1}{2} + \frac{1}{2}i, -\frac{1}{2} + \frac{1}{2}i\right)$, the corresponding singular value is $\|K\|_{\sigma,\mathbb{C}}=4$ (which coincides with singular value of the unfolding) and the bound is $\sqrt{hw}\|K\|_{\sigma,\mathbb{C}}=8$.

We have computed the spectral norm of the convolution with circular padding exactly for different input sizes using \cite{sedghi2018singular}. 
For input size $4\times 4$ (and sizes $4n \times 4n$) the spectral norm of the circular convolution is 8, which equals to the bound $\sqrt{hw}\|K\|_{\sigma,\mathbb{C}}$ and is, therefore, larger than $\sqrt{hw}\|K\|_{\sigma,\mathbb{R}}=4$.

This example illustrates the fact that in order to compute the bound $\sqrt{hw}\|K\|_{\sigma,\mathbb{C}}$ correctly (so that it upper bounds spectral norm of the convolution), we need to look for the best rank-1 approximation using the complex version of HOPM.

\section{Gradient computation \label{sec:C}}
Tensors $P_{real}, P_{im} \in \mathbb{R}^{2\times2\times2\times2}$ consist of $-1, 0, 1$. $P_{real}$ is equal to the real part of the tensor $(e_1+ie_2)^{\otimes 4}$, and $P_{im}$ is equal to its imaginary part. Here $e_1=(1, 0)^T, e_2=(0,1)^T$. The unfoldings of these tensors are as follows: 
\[P_{real}.reshape(2, 8, \texttt{order='c'})=\begin{bmatrix}
1 & 0 & 0 & -1 & 0 & -1 & -1 & 0\\
0 & -1 & -1 & 0 & -1 & 0 & 0 & 1\\
\end{bmatrix}\]
\[P_{im}.reshape(2, 8, \texttt{order='c'})=\begin{bmatrix}
0 & 1 & 1 & 0 & 1 & 0 & 0 & -1\\
1 & 0 & 0 & -1 & 0 & -1 & -1 & 0\\
\end{bmatrix}\]
Tensors $P_{real}$ and $P_{im}$ are fixed, so we can precompute them beforehand and use later. For a faster computation of the gradient, we only need to compute the gradients $\nabla_K real, \nabla_K im$ according to the \eqref{main:eq:grad}, because we already have the values $real$ and $im$ from the forward pass.
One can also use automatic differentiation instead of the derived formulas.

\section{Higher Dimensional Convolution \label{sec:D}}
\subsection*{Proof of \Cref{main:th:higher_dim}}
\begin{proof}
An unfolding $R=K_{(1,2 3\dots d+2)}$ of the kernel $K \in \mathbb{R}^{c_{out}\times c_{in} \times h_1 \times \ldots \times h_d}$ is a submatrix of a multi-level Toeplitz matrix $T$, hence, $\|R\|_2\leq \|T\|_2$. As is shown in \Cref{main:pr:unfolding}, the spectral norm of an unfolding upper bounds tensor norm, so we can write the lower bound:
\[\|K\|_\sigma \leq \|R\|_2 \leq \|T\|_2\]

To prove the upper bound, we first need to state that inequality $\|T\|_2 \leq \|F\|_2$ holds true for the multidimensional convolution. For the readers' convenience, here we present the proof of this fact, which directly follows \cite[Lemma 4]{yi2020asymptotic} and \cite[Section VI.B]{yi2020asymptotic}.

Following \cite[Section VI.B]{yi2020asymptotic}, the Jacobian $T$ can be written as
$$T = \sum_{[k_1]<n_1} \dots\sum_{[k_d]< n_d}[J^{k_1}_{n_1} \otimes \dots \otimes J^{k_d}_{n_d}] \otimes T_k, $$
$$T_{k} = \frac{1}{(2\pi)^d}\int_{\Omega}^{}F(\tau)e^{-i\langle k, \tau \rangle } d\tau,$$
where $J^k_n$ is a matrix of size $n\times n$ with ones along the $k^{th}$ diagonal and zeros elsewhere in the case of zero padding or $J^k_n=P^k$ in the case of circular padding ($P$ is a permutation matrix defined in \eqref{eq:permutation}). $\Omega=[-\pi, \pi]^d$, $k=(k_1, \dots, k_d), \tau=(\tau_1, \dots, \tau_d)$. We can write the generating function as
$$F(\tau) = \sum_{k} T_{k} e^{i\langle k, \tau \rangle }.$$
Let $u \in \mathbb{R}^{c_{out}n^d}, v \in \mathbb{R}^{c_{in}n^d}$ be the singular vectors of $T$. Let us divide $u$ and $v$ into $n^d$ subvectors of size $c_{out}$ and $c_{in}$. Let $u_{k} \in \mathbb{R}^{c_{out}}, v_{m} \in \mathbb{R}^{c_{in}}$ be the $k^{th}$ and $m^{th}$ subvectors corresponding to $(T)_{k, m} = T_{k - m}$, $1\leq m, k \leq n^d$, so we can write:
$$\begin{aligned}&\|T\|_2 = u^TTv=\sum_k \sum_{m} u_{k}^T T_{k - m} v_{m}=\sum_k \sum_m\frac{1}{(2\pi)^d}\int_{\Omega} u_{k}^T F(\tau)e^{-i\langle k - m, \tau \rangle }  v_{m}  d\tau =\\&{= \frac{1}{(2\pi)^d}\int_{\Omega}u(\tau)^TF(\tau)v(\tau) d\tau,}\end{aligned}$$
where
$$u(\tau)=\sum_k u_{k}e^{-i\langle k, \tau \rangle }, \quad  v(\tau)=\sum_m v_{m}e^{i\langle m, \tau \rangle }.$$
Following \cite[Lemma 4]{yi2020asymptotic}:
$$\begin{aligned}\|T\|_2 &\leq   \frac{1}{(2\pi)^d}\int_{\Omega}\|F\|_2\|u(\tau)\|_2\|v(\tau)\|_2 d\tau \leq  
\\&
\leq  \|F\|_2 \frac{1}{(2\pi)^d}\sqrt{\int_{\Omega}\|u(\tau)\|^2_2d\tau}\sqrt{\int_{\Omega}\|v(\tau)\|^2_2 d\tau }
 = \|F\|_2 \|u\|_2\|v\|_2= \|F\|_2.\end{aligned}$$
 Analogously to \cref{main:th:1}, 
$$\begin{aligned}&\|F\|_2 = \sup_{u_1, u_2}|\llbracket K; u_1, u_2, z_1 \dots z_{d}\rrbracket | =  \sqrt{h_1 \dots h_{d}} \sup_{u_1, u_2}|\llbracket K; u_1, u_2, \frac{z_1}{\sqrt{h_1}} \dots \frac{z_{d}}{\sqrt{h_{d}}}\rrbracket |\leq \\& \leq \sqrt{h_1 \dots h_{d}} \|K\|_{\sigma}.\end{aligned}$$
Thus, $\|T\|_2\leq \|F\|_2\leq \|K\|_{\sigma}$, which completes the proof.

\end{proof}

\section{Spectral norm regularization}
\cref{table_hyperparameters} presents accuracy for ResNet18 on CIFAR100 and ResNet34 on ImageNet with different regularizers and regularization coefficients $\beta$.

\begin{table}
\caption{\label{table_hyperparameters} Test accuracy with different regularizers.}
\begin{subtable}{.5\textwidth}
    \centering
    \begin{tabular}{c|ccc}
    \toprule
    Method & $\beta$ & Acc. w/o wd & Acc. w/ wd \\
    \midrule
    Baseline & 0 & 73.10 & 73.84\\
    $F4$  & 0.0016& 73.55 & 74.83 \\
    $F4$  & 0.0018& 73.69 & 74.44 \\
    $F4$  & 0.0022& \textbf{73.96} & 74.91 \\
    $TN$ (Ours) & 0.0016& 73.77&  74.76 \\
    $TN$ (Ours) & 0.0018& 73.63 & 74.77 \\
    $TN$ (Ours) & 0.0022& \textbf{73.96} & \textbf{74.99} \\
    \bottomrule
    \end{tabular}
    \caption{ResNet18 trained on CIFAR100}
\end{subtable}%
\hspace{2mm}
\begin{subtable}{.5\textwidth}
    \centering
    \begin{tabular}{c|ccc}
    \toprule
    Method & $\beta$ & Acc@1 & Acc@5 \\
    \midrule
    Baseline & 0 &  73.368 & \textbf{91.438}\\
    $F4$ & 2e-3 & 73.034 & 91.136  \\
    $F4$  & 1e-4 & 73.388 & 91.300  \\
    $F4$  & 5e-4 &  73.322 &  91.258 \\
    $TN$ (Ours) & 2e-3 &  73.086 & 91.150  \\
    $TN$ (Ours) & 1e-4 & \textbf{73.510} & 91.420 \\
    $TN$ (Ours) & 5e-4 & 73.372  & 91.326 \\
    \bottomrule
    \end{tabular}
    \caption{ResNet34 trained on ImageNet }
\end{subtable}
\end{table}
\vspace{-5mm}

\section{Spectral norm of convolutional layers of CNNs \label{sec:resnet}}

\Cref{fig:1} demonstrates the behaviour of the bounds during the training of ResNet-18, showing that our bound remains an accurate approximation of the exact spectral norm throughout the training process. We reinitialize the singular vectors every epoch from random approximation and use one iteration per training step to update the singular vectors. We observe that this strategy is good enough to maintain the precision during the training process.

To constrain the Lipschitz constant of an entire network, we should take into account spectral norm of each layer. Convolutional layers in CNNs are usually followed by batch normalization. Their concatenation forms a linear transformation for which we can estimate the spectral norm. In our experiments, we apply regularization only to the convolutional layers. \cref{fig:2} demonstrates that training with regularization noticeably decreases  the spectral norm of convolutions. In addition, our $TN$ bound becomes more accurate when regularization is applied. Although we do not regularize BatchNorm layers, \cref{fig:3} shows that the composition's spectral norm decreases when regularization is applied to convolutions.
Similarly to~\cite{delattre2023efficient}, in Figure~\ref{fig:all_methods}, we compare our method with the alternative approaches for different configurations of kernel tensors.

\begin{figure}
\centering
\begin{subfigure}{1.0\textwidth}
  \includegraphics[width=\linewidth]{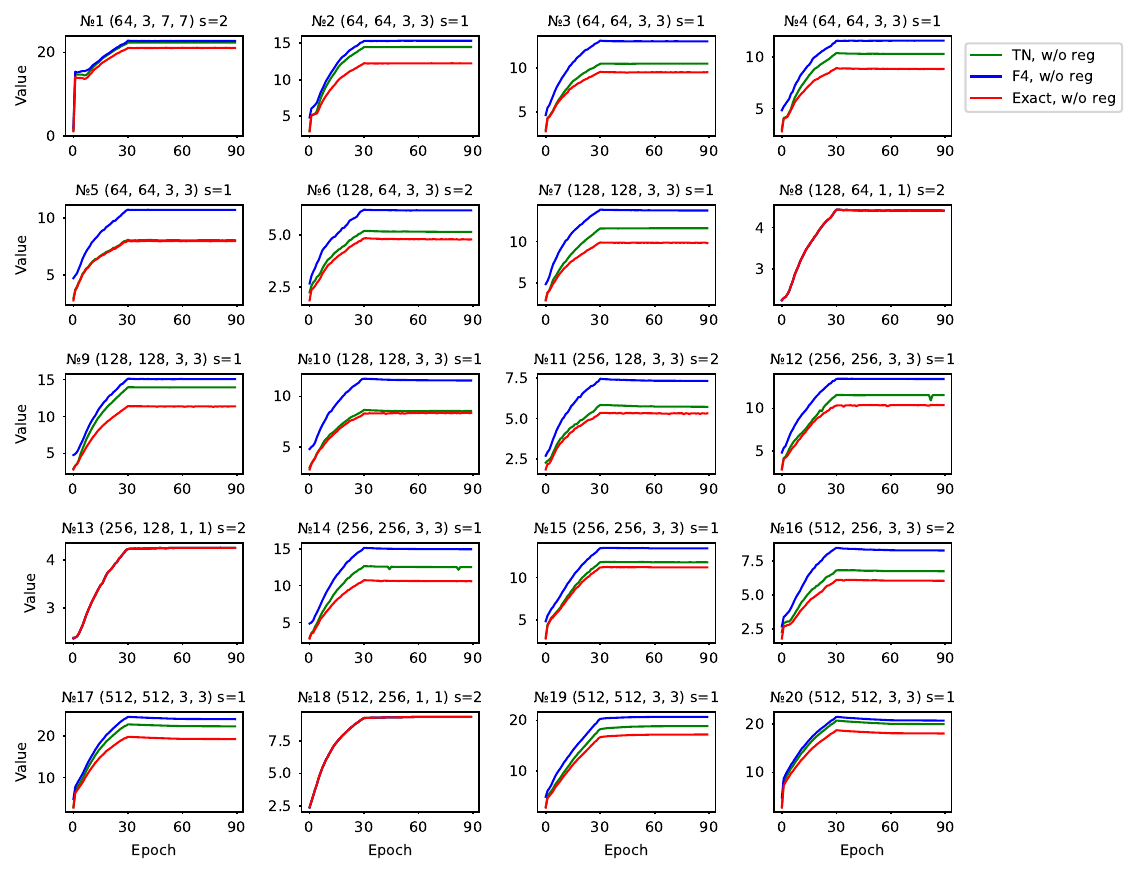}
\end{subfigure}
\caption{\label{fig:1} The plot compares our $TN$ bound with the $F4$ bound for convolutional layers of ResNet18 trained on CIFAR100. We do not use any regularization or weight decay in this experiment.}
\end{figure}

\begin{figure}
\centering
\begin{subfigure}{1.0\textwidth}
  \includegraphics[width=\linewidth]{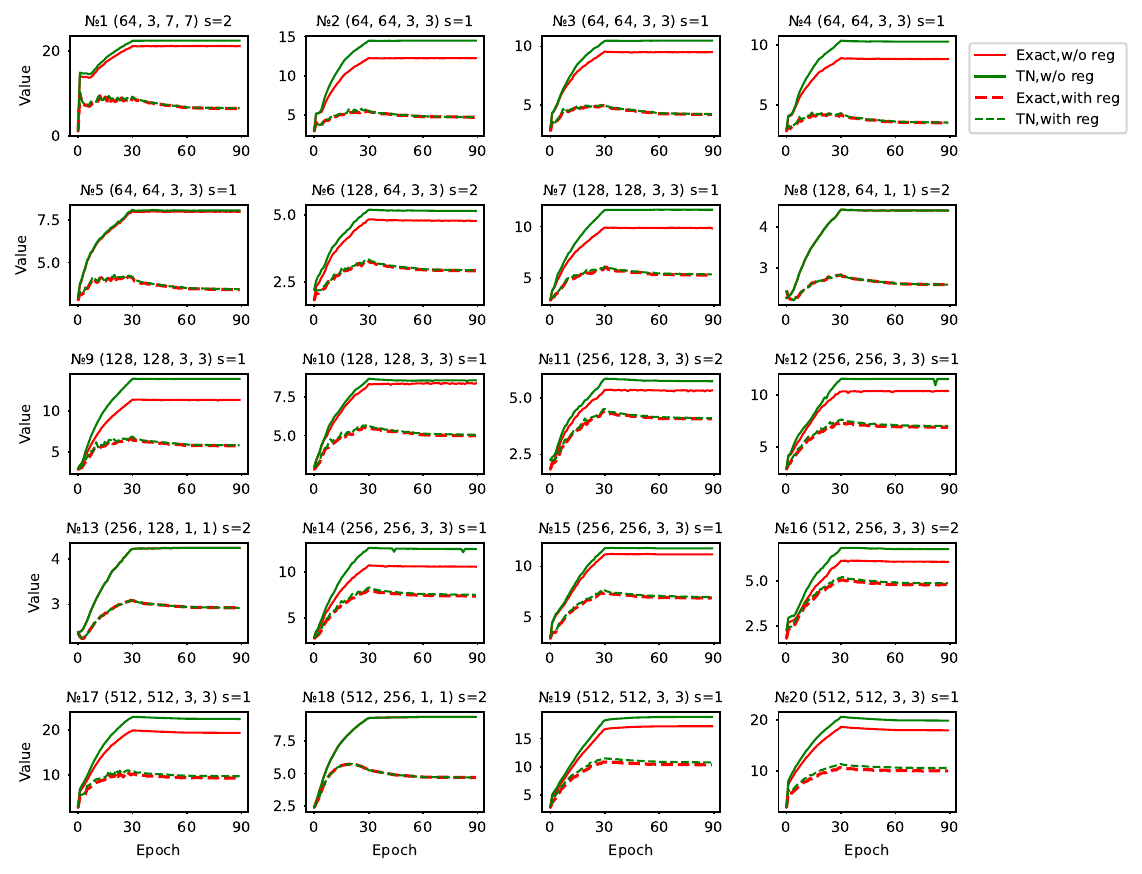}
\end{subfigure}
\caption{\label{fig:2} Effect of regularization with $TN$ bound on the spectral norm of convolutional layers of ResNet18 trained on CIFAR100. }
\end{figure}

\begin{figure}
\centering
\begin{subfigure}{1.0\textwidth}
  \includegraphics[width=\linewidth]{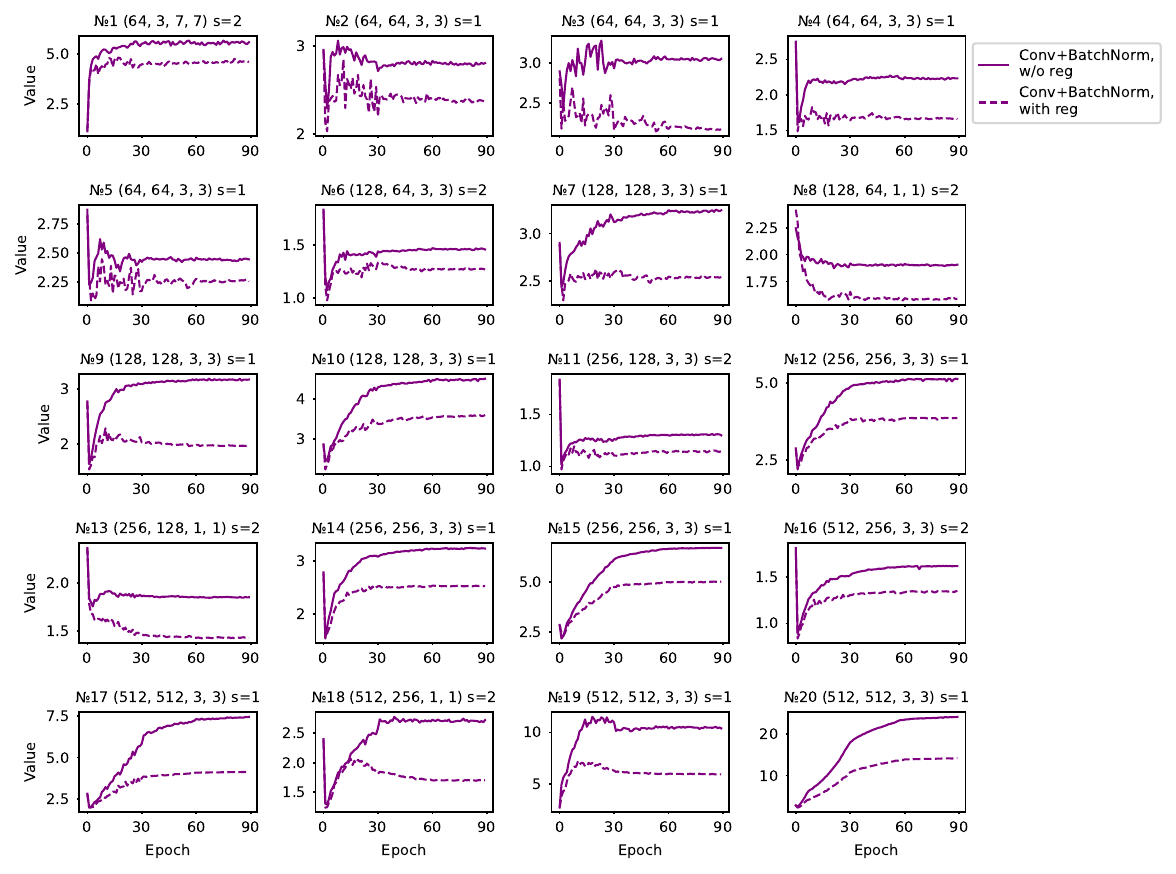}
\end{subfigure}
\caption{\label{fig:3} The behaviour of the spectral norm of composition of convolution and subsequent BatchNorm layers for ResNet18 trained on CIFAR100 with and without $TN$ regularization.}
\end{figure}

\begin{figure}\centering
\begin{subfigure}{1.0\textwidth}
  \includegraphics[width=\linewidth]{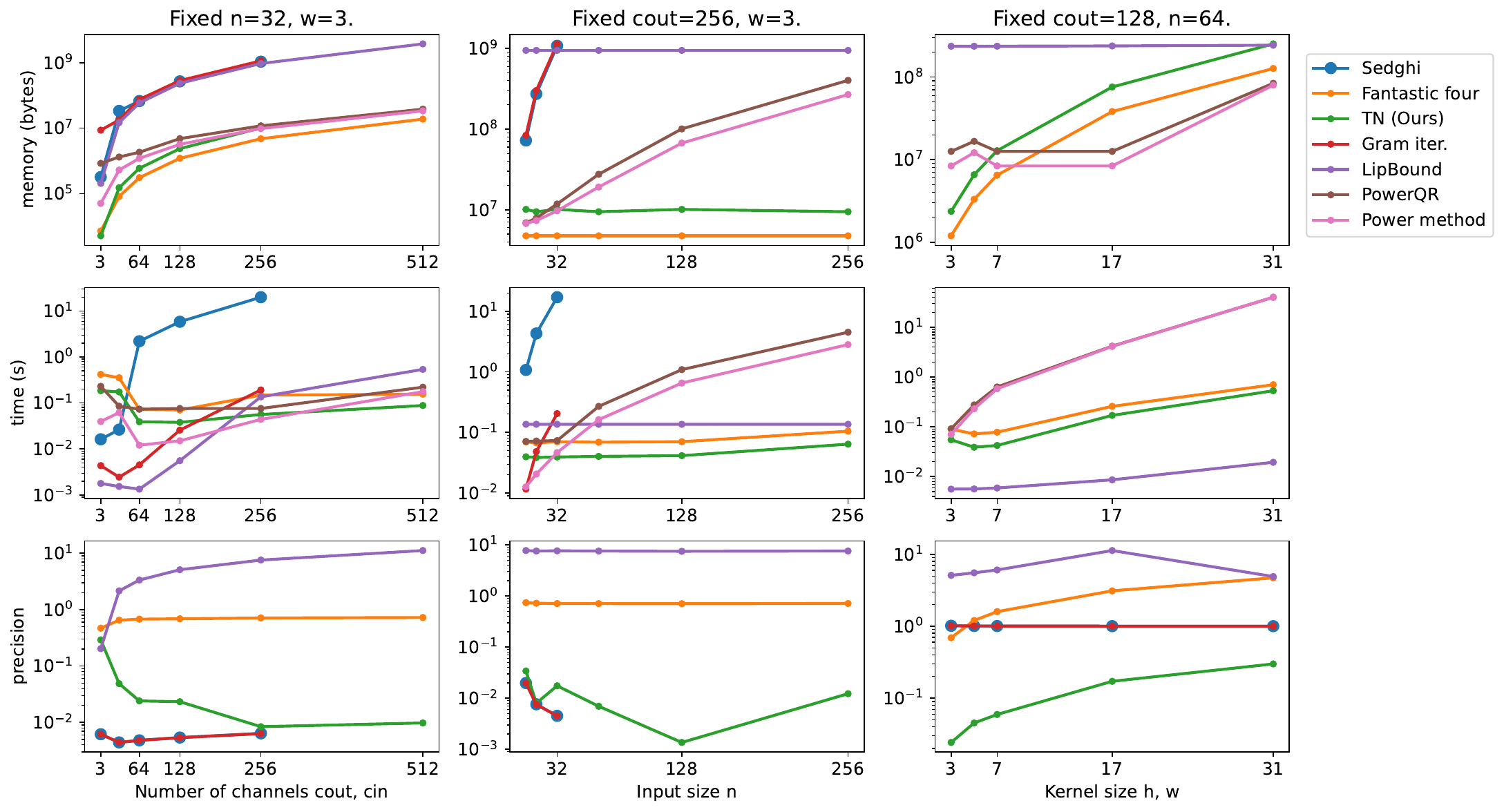}
\end{subfigure}
\caption{\label{fig:all_methods} Comparison of existing methods in terms of memory consumption, time efficiency and precision for  convolution with zero padding and kernels with entries sampled from $\mathcal{N}(0, 1)$. We measure the precision as $|\sigma_{method}-\sigma_{ref}| / \sigma_{ref}$, where $\sigma_{ref}$ is a highly accurate reference value obtained using the power method. We do not plot precision of PowerQR \cite{ebrahimpour2023spectrum} as it gives the exact value. The power method and PowerQR \cite{ryu2019plug, ebrahimpour2023spectrum} are accurate, but their time complexity noticeably depends on $n$ and~$c_{out}$. LipBound \cite{araujo2021lipschitz} produces errors larger than the other methods.
 Gram iteration \cite{delattre2023efficient} is fast, but consumes as much memory as the method by Sedghi \etal \cite{sedghi2018singular} and is inapplicable for large $c_{out}, c_{in}$ and $n$. Our method is memory efficient and provides a trade-off between speed and accuracy, improving the Fantastic four bound~\cite{singla2019fantastic}. }
\end{figure}
\end{document}